\documentclass{article}

\PassOptionsToPackage{numbers,sort&compress}{natbib}
\usepackage[final]{neurips_2026}

\usepackage[utf8]{inputenc}
\usepackage[T1]{fontenc}
\usepackage{hyperref}
\usepackage{url}
\usepackage{booktabs}
\usepackage{amsfonts,amsmath,amssymb,amsthm,mathtools}
\usepackage{nicefrac}
\usepackage{microtype}
\usepackage{xcolor}
\usepackage{graphicx}
\usepackage{multirow}
\usepackage[algo2e,ruled]{algorithm2e}
\usepackage[capitalize,noabbrev]{cleveref}

\theoremstyle{plain}
\newtheorem{theorem}{Theorem}
\newtheorem{lemma}[theorem]{Lemma}
\newtheorem{proposition}[theorem]{Proposition}
\newtheorem{corollary}[theorem]{Corollary}
\newtheorem{assumption}[theorem]{Assumption}
\crefname{corollary}{Corollary}{Corollaries}\Crefname{corollary}{Corollary}{Corollaries}
\crefname{assumption}{Assumption}{Assumptions}\Crefname{assumption}{Assumption}{Assumptions}
\newcommand{\key}[2]{{\color[HTML]{#1}\rule[0.5ex]{1.6em}{1.6pt}}\,#2}
\newcommand{\keydashed}[2]{{\color[HTML]{#1}\rule[0.5ex]{0.5em}{1.6pt}\hspace{0.2em}\rule[0.5ex]{0.5em}{1.6pt}}\,#2}
\newcommand{\keydotted}[2]{{\color[HTML]{#1}\rule[0.5ex]{0.25em}{1.6pt}\hspace{0.15em}\rule[0.5ex]{0.25em}{1.6pt}\hspace{0.15em}\rule[0.5ex]{0.25em}{1.6pt}\hspace{0.15em}\rule[0.5ex]{0.25em}{1.6pt}}\,#2}

\newcommand{\E}{\mathbb{E}}
\newcommand{\Prob}{\mathbb{P}}
\newcommand{\opt}{S^{\star}}

\newcommand{\CorrFull}{\textsc{CorrCTS-Full}}
\newcommand{\RobustCorr}{\textsc{RobustCorrCTS}}
\newcommand{\EdgePrune}{\textsc{CorrCTS-Prune}}
\newcommand{\KMRefine}{\textsc{CorrCTS-KM}}
\newcommand{\GPCTS}{\textsc{GP-CTS}}
\newcommand{\dtv}{d_{\mathrm{TV}}}

\title{Structure, Not Belief:\\Correlated Thompson Sampling from LLM-Derived\\Covariance in Combinatorial Semi-Bandits}

\author{%
  Vikram Kakaria\\
  Princeton University\\
  \And
  Anish Kataria\\
  Princeton University\\
  \And
  Anany Kotawala\\
  Princeton University\\
}

\begin{document}

\maketitle

\begin{abstract}
Combinatorial Thompson sampling (CTS) draws independent posterior samples for every arm, so its exploration dynamics ignore any relation among arms.
We study a minimal change to those dynamics: an LLM is queried \emph{once} for a partition of the arms, the partition becomes a positive-definite correlation matrix $\Sigma$ through an RBF kernel on cluster ranks, and the per-round posterior sample is drawn with covariance $\Sigma$ while the Beta posteriors are updated from real rewards only, so the LLM shapes \emph{how} the sampler moves, not \emph{what} it believes.
We give a self-contained Bayesian regret bound for the idealized Gaussian sampler whose information gain splits into a $K\log T$ term from the $K$-cluster structure and a ridge term that grows to $d\log T$: the $\sqrt{d/K}$ improvement over independent sampling is a finite-horizon transient, exact only as the within-cluster correlation tends to one.
The correlated sampler reduces regret by $19\%$ over CTS on 16 synthetic Bernoulli families at $T{=}2{,}500$ ($6$--$7\%$ at $T{=}25{,}000$ with data-adaptive kernels) and by $41\%$ on the Microsoft MIND-small news benchmark ($d{=}200$ real articles), while pseudo-observation warm starts give nothing.
An LLM-free ablation with a simulated oracle of controlled quality shows that on unstructured instances the gain is a property of the kernel \emph{shape} (a random partition, or a plain tempering of the sampling noise, reproduces it), while belief injection at matched oracle quality never helps.

\end{abstract}
\section{Introduction}\label{sec:intro}

Thompson sampling is a stochastic dynamical system: the learner's state is a posterior, the control is a random draw from it, and the reward feedback closes the loop; this paper is about how the covariance of that draw changes exploration, how the effect decays with the horizon, and how an oracle in the loop couples to the learner's state.
In the combinatorial semi-bandit \citep{chen2013cucb,kveton2015tight}, where $m$ of $d$ arms are played each round and each played arm returns its own Bernoulli reward, Combinatorial Thompson Sampling (CTS) \citep{wang2018cts} draws the $d$ coordinates of that control \emph{independently}, so its exploration dynamics carry no information about which arms are alike.

Large language models are an obvious source of such information, and most LLM-augmented bandits inject it as a \emph{belief}: a prior over means \citep{alamdari2024jumpstarting,bayley2025robustness}, pseudo-observations, or a recommendation mixed into the action distribution \citep{bouneffouf2025survey,krishnamurthy2024canllms}.
Belief injection has a feedback-loop failure mode that we observed directly (\cref{app:llmtables}): the LLM's answer quality depends on the estimates it is shown, which depend on the learner's own trajectory, and a wrong belief biases the posterior at a cost of real observations.
We instead use the LLM for \emph{structure}: one cluster query, converted to a correlation matrix $\Sigma$, changes only the covariance of the posterior draw, so posterior means and variances stay those of CTS (\cref{prop:integrity}) and an arbitrarily bad $\Sigma$ can distort the sampler's motion but not corrupt its state.

\paragraph{Contributions.}
(i)~A statement of what $\Sigma$ represents (\cref{ass:prior}) and a self-contained Bayesian regret bound for $\GPCTS(\Sigma)$ (\cref{thm:bayes}) whose $\sqrt{d/K}$ advantage is a transient, exact only as $\rho_{\max}\to1$.
(ii)~\CorrFull{} as a \emph{fixed-correlation, calibrated-marginal} approximation to $\GPCTS(\Sigma)$, partition-free accounts of why a fixed correlation changes top-$m$ selection (\cref{lem:slepian,lem:temper}; in the equicorrelated case it is exactly a tempering), and a total-variation guarantee for the credibility gate (\cref{prop:gate}).
(iii)~LLM experiments on synthetic families and real MIND click data, plus an LLM-free ablation in which every injection mode receives the \emph{same} simulated partition of controlled quality, separating kernel shape from partition information and structure from belief.

\section{Setting and algorithm}\label{sec:setup}

\paragraph{Combinatorial semi-bandit.}
There are $d$ arms with unknown means $\mu\in[0,1]^d$.
At round $t$ the learner plays $S_t\subseteq[d]$ with $|S_t|=m$ and observes $X_{i,t}\sim\mathrm{Bernoulli}(\mu_i)$ for each $i\in S_t$.
With $\opt=\arg\max_{|S|=m}\sum_{i\in S}\mu_i$, the regret is $R(T)=\sum_{t\le T}\bigl(\sum_{i\in\opt}\mu_i-\sum_{i\in S_t}\mu_i\bigr)$.
CTS \citep{wang2018cts} keeps $\mathrm{Beta}(\alpha_i,\beta_i)$ posteriors, draws $\theta_i$ independently, plays the top-$m$ of $\theta$, and updates $(\alpha_i,\beta_i)$ for $i\in S_t$.
At a warmup time $T_w$ the learner may send the empirical means $\hat\mu$ (and, for MIND, article titles) to an oracle that returns a partition $\mathcal C=\{C_1,\dots,C_K\}$ of $[d]$ of unknown quality, which may depend on $\hat\mu$.

\paragraph{From partition to covariance.}
Clusters are ranked by their mean $\hat\mu$ at $T_w$ (rank $0$ is best), each arm receives the normalized rank $f_i=r_{k(i)}/(K-1)\in[0,1]$ of its cluster, and
\begin{equation}\label{eq:rbf}
\Sigma_{ij}=\rho_{\max}\exp\!\bigl(-(f_i-f_j)^2/(2\ell^2)\bigr)\ (i\ne j),\qquad \Sigma_{ii}=1 .
\end{equation}
Writing $G$ for the RBF Gram matrix on the features $f$, $\Sigma=\rho_{\max}G+(1-\rho_{\max})I_d\succeq(1-\rho_{\max})I_d$ is positive definite for any partition, and $\mathrm{rank}(G)\le K$ because $f$ takes at most $K$ distinct values; same-cluster arms have correlation $\rho_{\max}$ and adjacent clusters keep positive correlation.

\paragraph{\CorrFull{} (pseudocode in \cref{app:alg}).}
Rounds $t\le T_w$ are plain CTS.
At $T_w$ the partition is turned into $\Sigma$ and its Cholesky factor $L$.
Afterwards each round takes the Beta posterior mean and standard deviation $(\bar\mu_i,\sigma_i)$ of every arm, draws $\epsilon\sim\mathcal N(0,I_d)$ and plays the top-$m$ of $\theta=\bar\mu+\sigma\odot(L\epsilon)$; the marginal of $\theta_i$ is the moment-matched Gaussian $\mathcal N(\bar\mu_i,\sigma_i^2)$, and only the joint law changes.
\RobustCorr{} adds a credibility weight $w\in[0,1]$ ($w=1$ initially), updated every $\Delta_{\mathrm{check}}$ rounds as $w\leftarrow0.7w+0.3\,\mathrm{overlap}_t$ ($\mathrm{overlap}_t$: fraction of the empirical top-$m$ in the oracle's best cluster), and plays the top-$m$ of $w\theta^{\mathrm{corr}}+(1-w)\theta^{\mathrm{ind}}$ (independent draws).
\EdgePrune{} and \KMRefine{} refresh $\Sigma$ from data at long horizons without new queries (\cref{app:variants}).

\begin{proposition}[Posterior integrity]\label[proposition]{prop:integrity}
For every reward sequence, the Beta parameters $(\alpha_i,\beta_i)$ of every arm coincide under \CorrFull{}, \RobustCorr{} and CTS; $\Sigma$ enters only the sampling step, and $\Sigma=I_d$ gives CTS with moment-matched Gaussian marginals (``CTS-Gaussian'').
\end{proposition}
\section{What a fixed correlation does to the sampling dynamics}\label{sec:theory}

\subsection{What $\Sigma$ stands for}
\begin{assumption}[Structural prior]\label[assumption]{ass:prior}
$\mu\sim\mathcal N(\mu_0\mathbf 1,\tau^2\Sigma)$ and, for $i\in S_t$, the observation is $y_{i,t}=\mu_i+\eta_{i,t}$ with $\eta_{i,t}\sim\mathcal N(0,\sigma^2)$ independent.
\end{assumption}
Thus $\Sigma$ is the prior correlation of the means and the oracle supplies an \emph{estimate} of it; when the estimate is wrong the bound below does not apply, the situation \cref{lem:slepian,lem:temper}, the credibility gate and \cref{sec:ablation} address.
$\GPCTS(\Sigma)$ denotes Thompson sampling under this model: $\theta_t\sim\mathcal N(\mu_t,\Sigma_t)$ from the exact Gaussian posterior, then top-$m$.

\subsection{Bayesian regret of $\GPCTS(\Sigma)$}
Let $D_T=\mathrm{diag}(n_1,\dots,n_d)$ count pulls after $T$ rounds ($\sum_i n_i=mT$) and let
$\gamma_T=\tfrac12\log\det\bigl(I_d+\sigma^{-2}\tau^2\,\Sigma^{1/2}D_T\Sigma^{1/2}\bigr)$
be the information gain of the play sequence.

\begin{lemma}[Two-term information gain]\label[lemma]{lem:gamma}
For $\Sigma=\rho G+(1-\rho)I_d$ with $G\succeq0$, $G_{ii}=1$, $\mathrm{rank}(G)\le K$, and $N=mT$,
\[
\gamma_T\ \le\ \frac K2\log\Bigl(1+\frac{\rho\tau^2N}{K\sigma^2}\Bigr)+\frac12\min\Bigl\{\frac{(1-\rho)\tau^2N}{\sigma^2},\ d\log\Bigl(1+\frac{(1-\rho)\tau^2N}{d\sigma^2}\Bigr)\Bigr\}.
\]
For any correlation matrix (e.g.\ $\Sigma=I_d$), $\gamma_T\le\frac d2\log(1+\tau^2N/(d\sigma^2))$; conversely, if every arm has been pulled at least $n$ times, $\gamma_T\ge\frac d2\log(1+(1-\rho)\tau^2n/\sigma^2)$, so the ridge term is not an artifact of the upper bound.
\end{lemma}

\begin{theorem}[Bayesian regret of $\GPCTS(\Sigma)$]\label[theorem]{thm:bayes}
Under \cref{ass:prior}, with $c=m\tau^2/\sigma^2$, $C_c=c/\log(1+c)$ and $\beta_T=2\log\bigl(\binom dm T^2\bigr)\le2m\log d+4\log T$,
\[
\E[R(T)]\ \le\ \sqrt{2\,C_c\,m\,\sigma^2\,\beta_T\,T\,\gamma_T}\;+\;\frac{m\tau}{T},
\]
with $\gamma_T$ bounded as in \cref{lem:gamma}; the proof (\cref{app:proofs}) follows \citet{russo2014learning} with an elliptical-potential step.
\end{theorem}
\paragraph{Reading the bound.}
With $\beta_T=O(m\log(dT))$ and the lemma,
\[
\E[R(T)]=O\Bigl(\sqrt{C_c}\,m\sqrt{T\log(dT)\,\bigl[K\log T+\min\{(1-\rho)\tau^2mT/\sigma^2,\ d\log T\}\bigr]}\Bigr).
\]
Because $\Sigma$ has full rank $d$ whenever $\rho<1$, the bound never depends on $K$ alone: the ridge term vanishes only as $\rho\to1$, where the rate becomes $\tilde O(m\sqrt{KT})$, a $\sqrt{d/K}$ improvement over the $\tilde O(m\sqrt{dT})$ that the same argument gives for $\Sigma=I_d$; for fixed $\rho<1$ the two rates coincide once $(1-\rho)\tau^2mT/\sigma^2\gtrsim K\log T$.
At the experimental values ($d{=}50,m{=}5,\rho{=}0.7,\tau^2{=}\nicefrac1{12},\sigma^2{=}\nicefrac14$) the crossover is below $150$ rounds (\cref{prop:crossover}), so a proper Bayesian sampler should gain little at $T=2{,}500$, as \cref{sec:ablation} confirms.

\subsection{\CorrFull{} versus $\GPCTS(\Sigma)$}
\CorrFull{} is not $\GPCTS(\Sigma)$: its marginals are the exact Beta moments, but its correlation stays fixed at $\Sigma$ whereas the Gaussian posterior correlation shrinks as data accumulate; a fixed positive correlation changes top-$m$ selection in two specific ways.

\begin{lemma}[Marginals kept, joint exceedance raised]\label[lemma]{lem:slepian}
Let $\theta\sim\mathcal N(\bar\mu,\mathrm{diag}(\sigma)\Sigma\,\mathrm{diag}(\sigma))$ for a correlation matrix $\Sigma$ with $\Sigma_{ij}\ge0$ for all $i,j$ (every kernel of \cref{eq:rbf}, for any partition).
Then each $\theta_i\sim\mathcal N(\bar\mu_i,\sigma_i^2)$ exactly as under $\Sigma=I_d$, and for every $A\subseteq[d]$ and threshold vector $c$, $\Prob_\Sigma(\theta_i>c_i\ \forall i\in A)\ge\prod_{i\in A}\Prob(\theta_i>c_i)$, with the left side nondecreasing in each off-diagonal entry of $\Sigma$.
\end{lemma}
This is Slepian's inequality \citep{slepian1962onesided} and holds for any partition: per-arm concentration is untouched, while the joint optimism event, which independent sampling grants $m$ near-optimal arms only with probability decaying like $q^m$, can only become more likely; correlation also shrinks the noise in the pairwise comparisons that top-$m$ selection is made of, and in the equicorrelated case this is \emph{exactly} a tempering.

\begin{lemma}[Comparison noise and tempering]\label[lemma]{lem:temper}
Let $\theta\sim\mathcal N(\bar\mu,\mathrm{diag}(\sigma)\Sigma\,\mathrm{diag}(\sigma))$.
(i) $\mathrm{Var}(\theta_i-\theta_j)=\sigma_i^2+\sigma_j^2-2\Sigma_{ij}\sigma_i\sigma_j$, which for $\sigma_i=\sigma_j$ equals the independent-sampling value times $1-\Sigma_{ij}$.
(ii) If $\Sigma=\rho\mathbf 1\mathbf 1^\top+(1-\rho)I_d$ and $\sigma_i\equiv\sigma$, then $\mathrm{top}\text{-}m(\theta)$ has exactly the law of $\mathrm{top}\text{-}m(\bar\mu+\sqrt{1-\rho}\,\sigma\epsilon)$ with $\epsilon\sim\mathcal N(0,I_d)$, i.e.\ of independent sampling tempered by $c=\sqrt{1-\rho}$.
\end{lemma}
With $\rho_{\max}=0.7$, $\sqrt{1-\rho_{\max}}\approx0.55$, and \cref{sec:ablation} finds the best tempering factor near $c=0.6$; \CorrFull{}, whose $\sigma_i$ differ across arms, is a heteroscedastic, kernel-shaped relaxation of (ii).

\subsection{The credibility gate}
\begin{proposition}[Gate convergence in total variation]\label[proposition]{prop:gate}
Given a common posterior, the per-round sampling distributions of \RobustCorr{} (correlation $\Sigma_w=w^2\Sigma+(1-w)^2I_d$ at credibility $w$) and CTS-Gaussian satisfy $\dtv\le\tfrac32\,w_t\,(\rho_{\max}d+2\sqrt d)$.
If the overlap statistic is $0$ at every check after $T_w$, then $\sum_{t>T_w}w_t\le\frac{10}{3}\Delta_{\mathrm{check}}$ and $\sum_{t>T_w}\dtv\le5(\rho_{\max}d+2\sqrt d)\Delta_{\mathrm{check}}$ for every $T$.
\end{proposition}
We claim no regret inequality (that would need a coupling of trajectories sharing a posterior only until their first disagreement): the guarantee is that under sustained disagreement the gate returns to CTS-Gaussian within $O(\Delta_{\mathrm{check}})$ rounds. Corralling with CTS \citep{agarwal2017corralling} is the principled alternative (\cref{app:alternatives}).

\section{Experiments}\label{sec:exp}

\subsection{LLM experiments}\label{sec:llm-exp}
All comparisons are paired (same instance, same seed) with Holm--Bonferroni correction \citep{holm1979simple}; $\rho_{\max}=0.7,\ \ell=0.5,\ K=8,\ T_w=30$ were chosen on a disjoint validation grid (\cref{app:details}).
\Cref{tab:llm} summarizes three experiments with an LLM partition (GPT-5.4 for E1--E2, Claude Sonnet~4.6 for E3; breakdowns in \cref{app:llmtables}).
\textbf{E1} (synthetic Bernoulli, $T{=}2{,}500$): $d\in\{25,50\}$, $m\in\{3,5\}$, four gap families, 20 seeds ($n{=}320$); gains grow with $d$ ($12$--$15\%$ at $d{=}25$, $19$--$25\%$ at $d{=}50$) and recommendation mixing collapses at $d{=}50,m{=}5$ (\cref{fig:curves}).
\textbf{E2} ($T{=}25{,}000$; the 16 E1 families plus 16 held-out ones, 15 seeds, $n{=}480$; \cref{app:llmtables}): CTS 615.9, static \CorrFull{} 587.9 ($+4.5\%$), \EdgePrune{} 578.5 ($+6.1\%$), \KMRefine{} 573.8 ($+6.8\%$), all $p<10^{-4}$; the static advantage decays and the adaptive variants overtake it, as \cref{thm:bayes} predicts.
\textbf{E3} (MIND-small \citep{wu2020mind}, $d{=}200$ articles with empirical CTRs, $m{=}5$, $T{=}10{,}000$, three user profiles $\times$ three article samples, 15 seeds, $n{=}135$): the reduction is stable across profiles ($40$--$42\%$) while warm start is flooded by real data.
Under adversarial clusters (true top-$m$ grouped with the $m$ worst arms, ranked first) \CorrFull{} loses $16.5\%$ to CTS on E1 while \RobustCorr{} loses $1.8\%$.

\begin{table}[!t]
\caption{LLM experiments E1 and E3: paired regret reduction vs CTS (CTS mean regret 330.7 on E1, 1573.7 on E3), paired wins/losses, and Holm-corrected $p$ (sign test for E1, Wilcoxon for E3).}
\label{tab:llm}
\centering\footnotesize
\renewcommand{\arraystretch}{0.85}
\begin{tabular}{@{}llrrcl@{}}
\toprule
Exp. & Algorithm & Regret & vs CTS & W/L & $p$ \\
\midrule
E1 ($n{=}320$) & \CorrFull{} & 268.2 & $+18.9\%$ & 262/58 & ${<}10^{-4}$ \\
 & Rec.-mixing (belief) & 279.5 & $+15.5\%$ & 143/97 & $0.0036$ \\
 & \RobustCorr{} & 284.2 & $+14.1\%$ & 237/83 & ${<}10^{-4}$ \\
 & RandomCorr (block, random partition) & 307.2 & $+7.1\%$ & 227/93 & ${<}10^{-4}$ \\

\midrule
E3 ($n{=}135$) & \CorrFull{} & 927.5 & $+41.1\%$ & 134/1 & ${<}10^{-6}$ \\
 & \KMRefine{} & 928.6 & $+41.0\%$ & 134/1 & ${<}10^{-6}$ \\
 & \EdgePrune{} & 943.7 & $+40.0\%$ & 133/2 & ${<}10^{-6}$ \\
 & Warm-start CTS (belief) & 1591.4 & $-1.1\%$ & 53/82 & $0.014$ \\
\bottomrule
\end{tabular}
\end{table}

\begin{figure}[!t]
\centering
\includegraphics[width=0.46\linewidth]{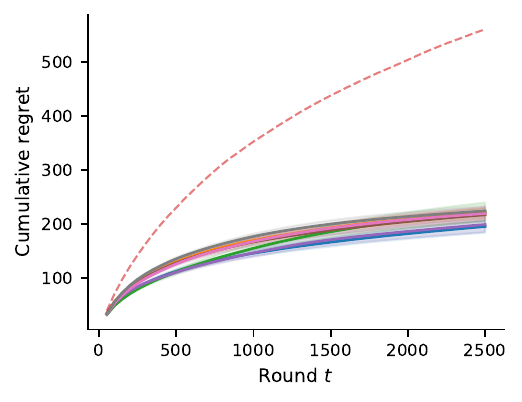}\hfill\includegraphics[width=0.46\linewidth]{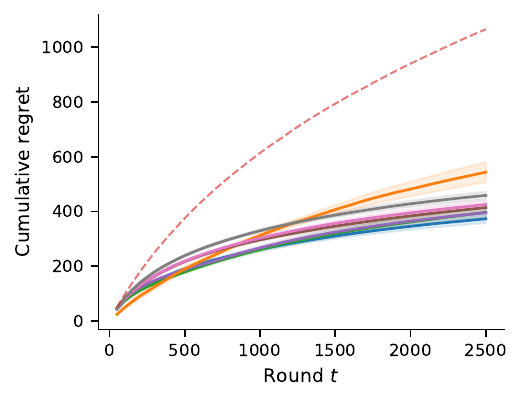}\\[2pt]
{\small \key{1f77b4}{\CorrFull{}}\quad \key{2ca02c}{\RobustCorr{}}\quad \key{ff7f0e}{Rec.-mixing}\quad \key{8c564b}{RandomCorr}\quad \key{e377c2}{block-diagonal}\quad \key{7f7f7f}{CTS}\quad \keydashed{d62728}{CUCB}}
\caption{E1 cumulative regret on the uniform-small family at $d{=}25$ (left) and $d{=}50$ (right), $m{=}5$ (20 seeds; bands: one standard error; purple: a gated block-diagonal pilot variant). The correlated samplers separate from CTS just after the query at $T_w{=}30$ and the gap widens with $d$; recommendation mixing competes only at $d{=}25$.}
\label{fig:curves}
\end{figure}

\subsection{Controlled ablation with a simulated oracle}\label{sec:ablation}
The RandomCorr row of \cref{tab:llm} changes the partition \emph{and} the kernel, so it cannot attribute the gain to partition quality; we therefore re-implemented the E1 family and all LLM-free samplers (\cref{tab:verify}) and replaced the LLM by a \emph{simulated oracle}: the quality-sorted partition into $K=8$ groups with each label replaced by a random one with probability $\varepsilon$.
All modes receive the same partition; the belief baselines inject its best-ranked cluster's $m$ best-$\hat\mu$ arms as $5$ pseudo-successes (warm start) or by recommendation mixing (\cref{fig:sweep}; \cref{app:abltables}).

\begin{figure}[!t]
\centering
\includegraphics[width=0.49\linewidth]{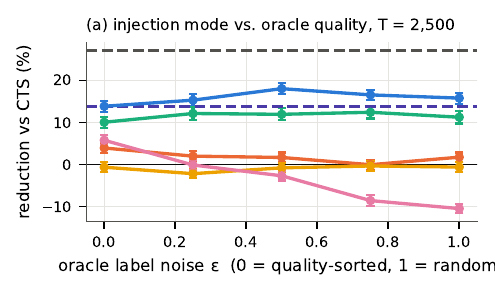}\hfill\includegraphics[width=0.49\linewidth]{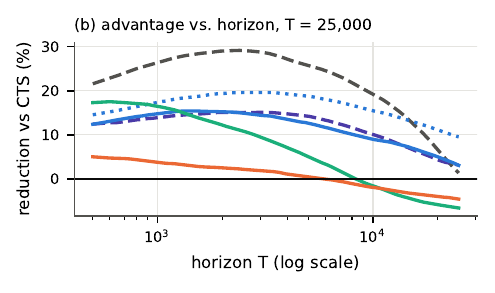}\\[2pt]
{\small \key{2a78d6}{\CorrFull{}}\quad \keydotted{2a78d6}{\CorrFull{}, $\varepsilon{=}0.5$}\quad \key{1baf7a}{\RobustCorr{}}\quad \key{eb6834}{$\GPCTS(\Sigma)$}\quad \key{eda100}{warm start}\quad \key{e87ba4}{rec.-mixing}\quad \keydashed{4a3aa7}{RBF, random partition}\quad \keydashed{52514e}{tempered, $c{=}0.6$}}
\caption{LLM-free ablation. (a) Paired regret reduction versus CTS against oracle label noise $\varepsilon$ (dashed: partition-free references; error bars: one standard error over 320 paired trials). (b) The same against the horizon (16 configurations, 10 seeds).}
\label{fig:sweep}
\end{figure}

\paragraph{Kernel shape versus partition information (\cref{tab:kernelshape}).}
A \emph{random} partition in \emph{random} order through the RBF kernel already gives $+13.8\%$ over CTS, the quality-sorted partition $+13.9\%$, and an adversarial partition that groups the true top-$m$ with the $m$ worst arms (keeping the $\hat\mu$ rank order) $+14.2\%$: on these unstructured instances the gain is a property of the kernel shape, not of the partition.

\paragraph{Structure versus belief at matched oracle quality (\cref{fig:sweep}a, \cref{tab:sweep}).}
Across $\varepsilon\in[0,1]$ the structure samplers stay between $+14\%$ and $+18\%$, warm start never differs from CTS, recommendation mixing falls from $+5.8\%$ to $-10.5\%$ as $\varepsilon$ goes from $0$ to $1$, and the exact Bayesian $\GPCTS(\Sigma)$ gains at most $+3.9\%$ (\cref{thm:bayes}).
\paragraph{Tempering and the horizon (\cref{fig:sweep}b, \cref{tab:kernelshape,tab:long}).}
The tempering reading of \cref{lem:temper} is confirmed by the simplest control: independent CTS-Gaussian with posterior standard deviation scaled by $c=0.6$ reduces regret by $+27.1\%$ at $T=2{,}500$, more than any kernel.
All are transient: at $T=25{,}000$ the RBF samplers retain $+3\%$, the tempered sampler $+1\%$, and $\GPCTS(\Sigma)$ and the gated sampler fall slightly below CTS.

\section{Discussion and limitations}\label{sec:discussion}
Moving LLM information from the posterior to the sampling covariance is safe by construction (\cref{prop:integrity}) and beneficial at moderate horizons; belief injection at matched oracle quality is not.
A \emph{fixed} correlation changes the exploration dynamics (\cref{lem:slepian,lem:temper}) whether or not the partition is right, so a tempered sampler is the LLM-free competitor to beat; partition \emph{information} buys a proper Bayesian sampler only a transient (\cref{thm:bayes}) that matters where latent structure exists (E3).
\emph{Limitations:} \cref{thm:bayes} concerns $\GPCTS(\Sigma)$ rather than \CorrFull{}; misspecification is covered only by \cref{lem:slepian,lem:temper,prop:gate}; E3 uses one dataset and one LLM.

\bibliographystyle{plainnat}
\bibliography{references}

\clearpage
\appendix
\raggedbottom
\makeatletter
\setlength{\@fptop}{0pt}\setlength{\@fpsep}{14pt plus 4pt}\setlength{\@fpbot}{0pt plus 1fil}
\makeatother

\section{Proofs}\label{app:proofs}

\subsection{Proof of \cref{lem:gamma}}
Write $\Sigma_1=\tau^2\Sigma$ and $M=\sigma^{-2}\Sigma_1^{1/2}D_T\Sigma_1^{1/2}$; then $\gamma_T=\tfrac12\log\det(I+M)$ and, since $\det(I+AB)=\det(I+BA)$, also $\gamma_T=\tfrac12\log\det(I+\sigma^{-2}D_T^{1/2}\Sigma_1D_T^{1/2})$.
Decompose $\sigma^{-2}D_T^{1/2}\Sigma_1D_T^{1/2}=A+B$ with $A=\sigma^{-2}\tau^2\rho\,D_T^{1/2}KD_T^{1/2}$ and $B=\sigma^{-2}\tau^2(1-\rho)D_T$, both positive semidefinite.

\emph{Step 1: $\det(I+A+B)\le\det(I+A)\det(I+B)$.}
$\det(I+A+B)=\det(I+A)\det(I+C)$ with $C=(I+A)^{-1/2}B(I+A)^{-1/2}$.
Since $(I+A)^{-1/2}$ is a contraction, $\lambda_j(C)\le\lambda_j(B)$ for every $j$ (Courant--Fischer), so $\det(I+C)\le\det(I+B)$.

\emph{Step 2: AM--GM.}
For a positive semidefinite matrix $M$ of rank at most $r$, $\log\det(I+M)=\sum_{j\le r}\log(1+\lambda_j)\le r\log(1+\mathrm{tr}(M)/r)$ by concavity of $\log$.
$A$ has rank at most $\mathrm{rank}(K)\le K$ and $\mathrm{tr}(A)=\sigma^{-2}\tau^2\rho\sum_in_iK_{ii}=\sigma^{-2}\tau^2\rho N$; $B$ is diagonal with trace $\sigma^{-2}\tau^2(1-\rho)N$ and rank at most $d$.
Combining the two steps gives the first claim (the $\min$ uses $\log(1+x)\le x$ on the second term); the second claim is Step~2 applied to the full matrix, whose trace is $\sigma^{-2}\tau^2N$ for any unit-diagonal $\Sigma$.
For the lower bound, $\Sigma\succeq(1-\rho)I_d$ gives $D_T^{1/2}\Sigma_1D_T^{1/2}\succeq\tau^2(1-\rho)D_T$, and $\log\det$ is monotone in the Loewner order, so $2\gamma_T\ge\sum_i\log(1+\sigma^{-2}\tau^2(1-\rho)n_i)\ge d\log(1+\sigma^{-2}\tau^2(1-\rho)n)$. \qed

\subsection{Proof of \cref{thm:bayes}}
Let $\mathcal A=\{S\subseteq[d]:|S|=m\}$, $f(S)=\mathbf 1_S^\top\mu$, and let $(\mu_t,\Sigma_t)$ be the Gaussian posterior mean and covariance of $\mu$ given the history $\mathcal H_t$ before round $t$ ($\Sigma_1=\tau^2\Sigma$).
Define the upper index $U_t(S)=\mathbf 1_S^\top\mu_t+\sqrt{\beta_T}\,\|\mathbf 1_S\|_{\Sigma_t}$, where $\|x\|_{\Sigma_t}^2=x^\top\Sigma_tx$.

\emph{Step 1: decomposition.}
Under Thompson sampling, $S_t$ and $\opt$ have the same conditional law given $\mathcal H_t$, and $U_t$ is $\mathcal H_t$-measurable, so $\E[U_t(\opt)\mid\mathcal H_t]=\E[U_t(S_t)\mid\mathcal H_t]$ \citep[Prop.~1]{russo2014learning}.
Hence
\[
\E[R(T)]=\E\sum_{t\le T}\bigl[f(\opt)-U_t(\opt)\bigr]+\E\sum_{t\le T}\bigl[U_t(S_t)-f(S_t)\bigr].
\]

\emph{Step 2: the first sum is small.}
Given $\mathcal H_t$, $\mathbf 1_S^\top(\mu-\mu_t)\sim\mathcal N(0,\|\mathbf 1_S\|_{\Sigma_t}^2)$.
For $Z\sim\mathcal N(0,s^2)$ and $b>0$, $\E[(Z-bs)_+]=s\int_b^\infty(z-b)\phi(z)\,dz\le s\,e^{-b^2/2}$.
Therefore $\E[(f(\opt)-U_t(\opt))_+\mid\mathcal H_t]\le\sum_{S\in\mathcal A}\E[(\mathbf 1_S^\top(\mu-\mu_t)-\sqrt{\beta_T}\|\mathbf 1_S\|_{\Sigma_t})_+\mid\mathcal H_t]\le|\mathcal A|\,e^{-\beta_T/2}\max_S\|\mathbf 1_S\|_{\Sigma_t}$.
Posterior variances never exceed prior variances, so $\|\mathbf 1_S\|_{\Sigma_t}^2\le m\sum_{i\in S}(\Sigma_t)_{ii}\le m^2\tau^2$, and with $\beta_T=2\log(|\mathcal A|T^2)$ the first sum is at most $T\cdot m\tau/T^2=m\tau/T$.

\emph{Step 3: the second sum.}
$S_t$ is a function of $\mathcal H_t$ and of sampling noise independent of $\mu$, so $\E[\mathbf 1_{S_t}^\top(\mu_t-\mu)\mid\mathcal H_t]=0$ and the second sum equals $\sqrt{\beta_T}\,\E\sum_t\|\mathbf 1_{S_t}\|_{\Sigma_t}\le\sqrt{\beta_T}\,\E\sqrt{T\sum_t\|\mathbf 1_{S_t}\|^2_{\Sigma_t}}$ by Cauchy--Schwarz.

\emph{Step 4: elliptical potential.}
Let $x_t=\sigma^{-2}\sum_{i\in S_t}(\Sigma_t)_{ii}\le m\tau^2/\sigma^2=c$.
Because $x\mapsto\log(1+x)/x$ is decreasing, $x_t\le C_c\log(1+x_t)$ with $C_c=c/\log(1+c)$.
The precision recursion $\Sigma_{t+1}^{-1}=\Sigma_t^{-1}+\sigma^{-2}\sum_{i\in S_t}e_ie_i^\top$ gives $\log\det\Sigma_{t+1}^{-1}-\log\det\Sigma_t^{-1}=\log\det(I+\sigma^{-2}\Sigma_t^{1/2}E_{S_t}\Sigma_t^{1/2})\ge\log(1+x_t)$, using $\det(I+M)\ge1+\mathrm{tr}(M)$ for $M\succeq0$.
Telescoping, $\sum_t\log(1+x_t)\le\log\det(\Sigma_1\Sigma_{T+1}^{-1})=2\gamma_T$.
Finally $\|\mathbf 1_{S_t}\|^2_{\Sigma_t}\le m\sum_{i\in S_t}(\Sigma_t)_{ii}=m\sigma^2x_t$, so $\sum_t\|\mathbf 1_{S_t}\|^2_{\Sigma_t}\le m\sigma^2C_c\cdot2\gamma_T$.
Combining Steps 1--4, $\E[R(T)]\le\sqrt{2C_cm\sigma^2\beta_TT\gamma_T}+m\tau/T$, where $\gamma_T$ may be replaced by the deterministic bound of \cref{lem:gamma} because that bound depends on the play sequence only through $N=mT$. \qed

\subsection{Proof of \cref{lem:slepian}}
The marginal claim is immediate since $\Sigma_{ii}=1$.
Restrict to the coordinates in $A$ and standardize: $Z_i=(\theta_i-\bar\mu_i)/\sigma_i$ is centered Gaussian with unit variances and correlations $\Sigma_{ij}$.
Slepian's inequality \citep{slepian1962onesided} states that for two centered Gaussian vectors with equal variances and $\Sigma^{(1)}_{ij}\le\Sigma^{(2)}_{ij}$ for all $i\ne j$, $\Prob(Z^{(1)}_i\le c_i\ \forall i)\le\Prob(Z^{(2)}_i\le c_i\ \forall i)$ for every $c$; applied to $-Z^{(1)},-Z^{(2)}$ (same covariances) it gives the same monotonicity for $\{Z_i>c_i\ \forall i\}$.
Taking $\Sigma^{(1)}=I_{|A|}$ and $\Sigma^{(2)}=\Sigma_{AA}$, which is admissible because $\Sigma_{ij}\ge0$, yields the product lower bound; monotonicity in each entry follows from the same comparison. \qed

\subsection{Proof of \cref{lem:temper}}
(i) is the bilinearity of covariance: $\mathrm{Var}(\theta_i-\theta_j)=\sigma_i^2+\sigma_j^2-2\,\mathrm{Cov}(\theta_i,\theta_j)$ and $\mathrm{Cov}(\theta_i,\theta_j)=\sigma_i\sigma_j\Sigma_{ij}$; under independence the same expression has $\Sigma_{ij}=0$, and for $\sigma_i=\sigma_j=\sigma$ the ratio is $2\sigma^2(1-\Sigma_{ij})/(2\sigma^2)$.
(ii) An equicorrelated vector has the representation $\theta=\bar\mu+\sigma(\sqrt\rho\,Z\mathbf 1+\sqrt{1-\rho}\,\epsilon)$ with $Z\sim\mathcal N(0,1)$ independent of $\epsilon\sim\mathcal N(0,I_d)$, since $\mathrm{Cov}(\theta)=\sigma^2(\rho\mathbf 1\mathbf 1^\top+(1-\rho)I_d)$.
The term $\sigma\sqrt\rho\,Z\mathbf 1$ adds the same number to every coordinate and therefore leaves the ordering of the coordinates, hence the top-$m$ set, unchanged almost surely (ties have probability zero).
Thus $\mathrm{top}\text{-}m(\theta)=\mathrm{top}\text{-}m(\bar\mu+\sqrt{1-\rho}\,\sigma\epsilon)$ pathwise, which is the claim. \qed

\subsection{Two consequences of \cref{lem:gamma,thm:bayes}}
\begin{corollary}[Independent prior]\label[corollary]{cor:indep}
For $\Sigma=I_d$ (CTS-Gaussian under \cref{ass:prior}), $\E[R(T)]\le\sqrt{C_c\,m\,\sigma^2\,\beta_T\,T\,d\log(1+\tau^2mT/(d\sigma^2))}+m\tau/T$.
\end{corollary}
\begin{proof}
Insert the second claim of \cref{lem:gamma} into \cref{thm:bayes}.
\end{proof}

\begin{proposition}[Crossover horizon]\label[proposition]{prop:crossover}
Let $T^\star=\sup\{T:\,(1-\rho)\tau^2mT/\sigma^2\le K\log(1+\rho\tau^2mT/(K\sigma^2))\}$.
For every $T\le T^\star$, $\gamma_T\le K\log(1+\rho\tau^2mT/(K\sigma^2))$, so the bound of \cref{thm:bayes} reads $O(\sqrt{C_c}\,m\sqrt{KT\log(dT)\log T})$, and
\[
T^\star\ \le\ \frac{4\rho K\sigma^2}{(1-\rho)^2\tau^2m}.
\]
\end{proposition}
\begin{proof}
For $T\le T^\star$ the second term of \cref{lem:gamma} is at most $\tfrac12(1-\rho)\tau^2mT/\sigma^2\le\tfrac K2\log(1+\rho\tau^2mT/(K\sigma^2))$, which gives the first claim.
For the second, write $x=(1-\rho)\tau^2mT^\star/\sigma^2$; the defining inequality is $x\le K\log\bigl(1+\tfrac{\rho}{(1-\rho)K}x\bigr)$, and $\log(1+y)\le2\sqrt y$ gives $x\le2K\sqrt{\rho x/((1-\rho)K)}$, i.e.\ $\sqrt x\le2\sqrt{\rho K/(1-\rho)}$, i.e.\ $x\le4\rho K/(1-\rho)$; solving for $T^\star$ gives the bound.
At the experimental values ($d{=}50$, $m{=}5$, $K{=}8$, $\rho{=}0.7$, $\tau^2{=}\nicefrac1{12}$, $\sigma^2{=}\nicefrac14$) the bound is $T^\star\le149$ rounds.
\end{proof}

\subsection{Proof of \cref{prop:gate}}
Given a common posterior, the two draws are $\bar\mu+\mathrm{diag}(\sigma)\xi$ with $\xi\sim\mathcal N(0,\Sigma_w)$ and $\xi\sim\mathcal N(0,I_d)$ respectively; total variation is invariant under the common affine map, so $\dtv=\dtv(\mathcal N(0,\Sigma_w),\mathcal N(0,I_d))\le\tfrac32\|\Sigma_w-I_d\|_F$ by \citet[Thm.~1.1]{devroye2018total}.
Now $\Sigma_w-I_d=w^2(\Sigma-I_d)-(2w-w^2)I_d$, $\|\Sigma-I_d\|_F\le\rho_{\max}\sqrt{d(d-1)}\le\rho_{\max}d$ and $\|(2w-w^2)I_d\|_F\le2w\sqrt d$, giving $\dtv\le\tfrac32w(\rho_{\max}d+2\sqrt d)$ since $w\le1$.
Under zero overlap at every check, $w$ equals $0.7^k$ during the $\Delta_{\mathrm{check}}$ rounds following the $k$-th check, so $\sum_tw_t\le\Delta_{\mathrm{check}}\sum_{k\ge0}0.7^k=\frac{10}{3}\Delta_{\mathrm{check}}$. \qed

\section{Pseudocode}\label{app:alg}
\begin{algorithm2e}[H]
\caption{\CorrFull{}}\label{alg:corrfull}
\SetAlgoLined
\KwIn{$d,m$, warmup $T_w$, kernel parameters $\rho_{\max},\ell$, number of clusters $K$}
$\alpha_i\gets1,\ \beta_i\gets1$ for all $i$; $L\gets I_d$\;
\For{$t=1,\dots,T$}{
  \eIf{$t\le T_w$}{$\theta_i\sim\mathrm{Beta}(\alpha_i,\beta_i)$ independently\;}{
    \If{$t=T_w+1$}{$\mathcal C\gets\textsc{OracleClusterQuery}(\hat\mu,K)$; build $\Sigma$ from \cref{eq:rbf}; $L\gets\mathrm{Cholesky}(\Sigma+10^{-3}I_d)$\;}
    $\theta\gets\bar\mu+\sigma\odot(L\epsilon)$ with $\epsilon\sim\mathcal N(0,I_d)$\;
  }
  play $S_t=\text{top-}m(\theta)$; observe $X_{i,t}$; $\alpha_i\gets\alpha_i+X_{i,t}$, $\beta_i\gets\beta_i+1-X_{i,t}$ for $i\in S_t$\;
}
\end{algorithm2e}

\section{Alternatives: $\GPCTS(\Sigma)$ directly, and corralling}\label{app:alternatives}
$\GPCTS(\Sigma)$ (as run in \cref{sec:ablation}) uses prior $\mathcal N(\tfrac12\mathbf 1,\tau^2\Sigma)$ with $\tau^2=\nicefrac1{12}$ (the variance of a $\mathrm{Beta}(1,1)$ mean) and $\sigma^2=\nicefrac14$ (the Bernoulli variance bound), a CTS warmup of $T_w$ rounds, and thereafter the exact Gaussian posterior computed from all counts.
It is the algorithm \cref{thm:bayes} is about.
Its posterior correlation between two arms decays at rate $1/n$ in their pull counts, which is why the structural effect is transient; \CorrFull{} freezes the correlation and thereby trades Bayesian coherence for a persistent change in the sampling dynamics.
A corralled combination \citep{agarwal2017corralling} of CTS and $\GPCTS(\Sigma)$ would give a consistency--robustness guarantee of the learning-augmented type \citep{lykouris2018competitive}: regret within $\tilde O(\sqrt T)$ of the better base algorithm.
The credibility gate of \RobustCorr{} is a heuristic with the same purpose and only the total-variation guarantee of \cref{prop:gate}; corralling is the right baseline for a follow-up and was not benchmarked here.

\section{Data-adaptive variants}\label{app:variants}
\EdgePrune{} keeps the initial kernel and, every $500$ rounds, zeroes $\Sigma_{ij}$ whenever $|\hat\mu_i-\hat\mu_j|>2\sqrt{1/n_i+1/n_j}$, then refactorizes.
\KMRefine{} blends the initial kernel with a kernel built from $k$-means on $\hat\mu$ at $t\in\{2{,}000,10{,}000\}$ with weight $w(t)=\max(0.2,1-t/12{,}000)$ on the LLM kernel.
Both make the decay of the correlation part of the dynamics rather than leaving it fixed.

\section{Experimental details}\label{app:details}
\paragraph{E1 instances.}
For $(d,m)\in\{25,50\}\times\{3,5\}$: \emph{uniform-large} (top $m$ at $0.7$, rest $0.55$), \emph{uniform-small} (rest $0.62$), \emph{hard-gap} (top $m$ at $0.7$, next $m$ at $0.65$, rest $0.5$), \emph{staggered} (means linearly spaced from $0.7$ to $0.3$); $T=2{,}500$, 20 seeds.
\paragraph{E2.} The 16 E1 families plus 16 with fresh environment seeds; $T=25{,}000$, 15 seeds.
\paragraph{E3.} MIND-small \citep{wu2020mind}, $17{,}226$ articles with empirical CTRs; each instance samples $d=200$ articles; three user profiles (``diverse'', ``entertainment'', ``tech'') reweight CTRs by category with weights fixed before any run; environment seeds $\{100,200,300\}$; $T=10{,}000$, 15 seeds; oracle Claude Sonnet~4.6 (AWS Bedrock, temperature $0.3$) receiving titles and $\hat\mu$ and asked for $K=8$ topic clusters.
\paragraph{Hyperparameters.} $T_w=30$, $K=8$, $\rho_{\max}=0.7$, $\ell=0.5$; \RobustCorr{} $\Delta_{\mathrm{check}}=100$; RandomCorr within-cluster $\rho=0.6$; recommendation mixing $\tau=100$, re-query every 100 rounds, one pseudo-success per recommended arm at each query; warm start $5$ pseudo-successes.
$\rho_{\max},\ell,K,T_w$ were selected on $(d,m)\in\{(25,3),(50,5)\}\times\{\text{uniform-large},\text{hard-gap}\}$ with environment seeds $700$--$704$, disjoint from all test seeds.
Sensitivity: reductions on E1 stay within $15$--$20\%$ for $\rho_{\max}\in[0.5,0.8]$ and $\ell\in[0.5,1.0]$; $\rho_{\max}=0.9$ and $\ell=0.25$ or $2.0$ lose $2$--$4$ points; $T_w\in\{10,100\}$ loses $2$--$4$ points; $K\in\{6,\dots,12\}$ is flat.
\paragraph{Simulated-oracle ablation.}
Same instances and seeds as E1; oracle label noise $\varepsilon\in\{0,\nicefrac14,\nicefrac12,\nicefrac34,1\}$; the adversarial partition places the true top-$m$ and the $m$ worst arms in one cluster and splits the remaining arms into $K-1$ quality-sorted groups, with cluster order by $\hat\mu$ as in \cref{eq:rbf} (in the E1 adversarial test the merged cluster was additionally forced to rank first); all samplers share the CTS warmup and the kernel of \cref{eq:rbf} with the E1 hyperparameters.
The tempered control draws $\theta=\bar\mu+c\,\sigma\odot\epsilon$ with independent $\epsilon$.
Each trial takes $0.02$--$0.2$\,s on one CPU core; the full sweep is about two minutes on ten cores.
\paragraph{Compute.} E1--E3 ran on a 16-core machine ($3$\,s, $25$\,s and $12$\,s per trial respectively); total LLM cost about \$12; LLM responses were cached.

\section{Tables for the LLM-free ablation}\label{app:abltables}
\Cref{tab:verify} checks the re-implementation, \cref{tab:kernelshape} is the kernel-shape ablation with the tempered controls, \cref{tab:sweep} tabulates \cref{fig:sweep}a, \cref{tab:long} tabulates \cref{fig:sweep}b, and \cref{tab:abl-perdm} gives the per-$(d,m)$ breakdown.
\begin{table}[htbp]
\caption{Per-$(d,m)$ mean regret in the LLM-free ablation ($T{=}2{,}500$; 4 gap families $\times$ 20 seeds $=80$ trials per cell).}
\label{tab:abl-perdm}
\centering\small
\begin{tabular}{@{}lcccc@{}}
\toprule
Sampler & $(25,3)$ & $(25,5)$ & $(50,3)$ & $(50,5)$ \\
\midrule
CTS & 256.7 & 284.4 & 442.2 & 492.6 \\
RBF kernel, random partition/order & 221.4 & 257.2 & 363.9 & 429.8 \\
\CorrFull{} ($\varepsilon{=}0$) & 225.0 & 247.3 & 364.7 & 434.3 \\
\CorrFull{} ($\varepsilon{=}0.5$) & 223.4 & 233.8 & 338.9 & 413.4 \\
\CorrFull{} ($\varepsilon{=}1$) & 221.7 & 240.5 & 353.2 & 427.3 \\
GP-CTS($\Sigma$) ($\varepsilon{=}0$) & 250.6 & 257.2 & 421.5 & 488.4 \\
\RobustCorr{} ($\varepsilon{=}0$) & 250.6 & 247.9 & 390.3 & 438.6 \\
Warm start ($\varepsilon{=}0$) & 256.1 & 280.4 & 452.7 & 496.1 \\
Rec.-mixing ($\varepsilon{=}0$) & 254.2 & 239.1 & 431.2 & 466.0 \\
Tempered CTS-Gaussian ($c{=}0.6$) & 196.8 & 241.3 & 281.6 & 356.5 \\
\bottomrule
\end{tabular}
\end{table}

\begin{table}[htbp]
\caption{Re-implementation check of the LLM-free rows of Experiment~1 (16 configurations $\times$ 20 seeds, $T{=}2{,}500$). ``Reported'' are the numbers from the LLM experiment harness; ``re-run'' are from an independent re-implementation of the same instance family and algorithms. Cells are mean cumulative regret; the four right-hand columns are the per-$(d,m)$ means.}
\label{tab:verify}
\centering\small
\begin{tabular}{@{}llcccccc@{}}
\toprule
Algorithm & Source & Overall & $(25,3)$ & $(25,5)$ & $(50,3)$ & $(50,5)$ \\
\midrule
CTS & reported & 330.7 & 224.5 & 223.8 & 415.7 & 458.8 \\
 & re-run & 369.0 & 256.7 & 284.4 & 442.2 & 492.6 \\
CUCB & reported & 776.3 & 546.4 & 560.6 & 933.2 & 1065.2 \\
 & re-run & 735.8 & 539.0 & 583.4 & 800.2 & 1020.7 \\
RandomCorr (block, $\rho{=}0.6$) & reported & 307.2 & 210.9 & 216.7 & 386.8 & 414.4 \\
 & re-run & 351.2 & 253.2 & 261.4 & 407.8 & 482.4 \\
\bottomrule\end{tabular}\end{table}

\begin{table}[htbp]
\caption{Kernel-shape versus partition-information ablation (LLM-free; $n{=}320$ paired trials per row, $T{=}2{,}500$). All RBF rows use the \CorrFull{} kernel with identical $(\rho_{\max},\ell,K)=(0.7,0.5,8)$ and differ only in what the partition and cluster order contain. ``vs CTS'' is the paired mean regret reduction; W/L are paired wins/losses; $p$ is a two-sided Wilcoxon signed-rank test.}
\label{tab:kernelshape}
\centering\footnotesize
\resizebox{\linewidth}{!}{%
\begin{tabular}{@{}lccccc@{}}
\toprule
Sampler & Partition & Cluster order & Regret & vs CTS & W/L ($p$) \\
\midrule
CTS & --- & --- & 369.0 $\pm$ 6.6 & --- & --- \\
Block-diagonal (RandomCorr) & random & none & 351.2 $\pm$ 6.7 & +4.8\% & 202/118 (${<}10^{-4}$) \\
RBF kernel & random & random & 318.1 $\pm$ 6.6 & +13.8\% & 239/81 (${<}10^{-4}$) \\
RBF kernel & random & by $\hat\mu$ & 309.9 $\pm$ 6.4 & +16.0\% & 248/72 (${<}10^{-4}$) \\
RBF kernel & oracle, $\varepsilon{=}1$ & by $\hat\mu$ & 310.7 $\pm$ 6.6 & +15.8\% & 250/70 (${<}10^{-4}$) \\
RBF kernel & oracle, $\varepsilon{=}0.5$ & by $\hat\mu$ & 302.4 $\pm$ 6.2 & +18.0\% & 252/68 (${<}10^{-4}$) \\
RBF kernel & oracle, $\varepsilon{=}0$ & by $\hat\mu$ & 317.8 $\pm$ 6.7 & +13.9\% & 239/81 (${<}10^{-4}$) \\
RBF kernel & adversarial & by $\hat\mu$ & 316.4 $\pm$ 6.6 & +14.2\% & 242/78 (${<}10^{-4}$) \\
RBF + credibility gate & adversarial & by $\hat\mu$ & 338.3 $\pm$ 6.7 & +8.3\% & 208/112 (${<}10^{-4}$) \\
\midrule
Tempered CTS-Gaussian ($c{=}0.8$) & none & none & 302.8 $\pm$ 6.1 & +17.9\% & 247/73 (${<}10^{-4}$) \\
Tempered CTS-Gaussian ($c{=}0.6$) & none & none & 269.0 $\pm$ 7.1 & +27.1\% & 273/47 (${<}10^{-4}$) \\
Tempered CTS-Gaussian ($c{=}0.4$) & none & none & 316.0 $\pm$ 11.0 & +14.4\% & 220/100 (${<}10^{-4}$) \\
\bottomrule\end{tabular}}\end{table}

\begin{table}[htbp]
\caption{Horizon dependence (LLM-free; 16 configurations $\times$ 10 seeds, $T{=}25{,}000$). Paired regret reduction versus CTS at intermediate horizons, read off the cumulative-regret curves of the same runs; the last column is the final W/L.}
\label{tab:long}
\centering\small
\begin{tabular}{@{}lrrrrrc@{}}
\toprule
Sampler & $T{=}2.5$k & $5$k & $10$k & $20$k & $25$k & W/L \\
\midrule
RBF kernel, random partition/order & +15.1\% & +14.2\% & +10.1\% & +4.3\% & +3.1\% & 114/46 \\
\CorrFull{} ($\varepsilon{=}0$) & +14.9\% & +12.6\% & +9.0\% & +5.0\% & +3.1\% & 113/47 \\
\CorrFull{} ($\varepsilon{=}0.5$) & +19.6\% & +18.8\% & +15.5\% & +11.1\% & +9.6\% & 122/38 \\
GP-CTS($\Sigma$) ($\varepsilon{=}0$) & +2.2\% & +0.5\% & -1.9\% & -4.0\% & -4.6\% & 65/94 \\
\RobustCorr{} ($\varepsilon{=}0$) & +10.8\% & +5.1\% & -1.6\% & -5.8\% & -6.6\% & 82/78 \\
Tempered CTS-Gaussian ($c{=}0.6$) & +29.1\% & +25.8\% & +19.2\% & +6.6\% & +1.3\% & 112/48 \\
\bottomrule\end{tabular}\end{table}

\begin{table}[htbp]
\caption{Structure versus belief injection at matched oracle quality (LLM-free; $n{=}320$ paired trials per cell, $T{=}2{,}500$). Every column receives the \emph{same} simulated partition; $\varepsilon$ is the fraction of arms whose cluster label is replaced by a uniformly random one ($\varepsilon{=}0$: quality-sorted partition; $\varepsilon{=}1$: random). NMI is against the top-$m$/rest partition (for reference, GPT-5.4 clusters in Experiment~1 had NMI $0.155$; random $0.10$). Cells are paired regret reduction versus CTS (positive is better); W/L in brackets.}
\label{tab:sweep}
\centering\footnotesize
\resizebox{\linewidth}{!}{%
\begin{tabular}{@{}lcccccc@{}}
\toprule
$\varepsilon$ & NMI & \CorrFull{} & GP-CTS($\Sigma$) & \RobustCorr{} & Warm-start (belief) & Rec.-mixing (belief) \\
\midrule
0.00 & 0.218 & +13.9\% [239/81] & +3.9\% [189/131] & +10.1\% [216/104] & -0.6\% [155/165] & +5.8\% [210/110] \\
0.25 & 0.172 & +15.3\% [249/71] & +2.0\% [175/145] & +12.1\% [224/96] & -2.2\% [143/177] & -0.1\% [171/149] \\
0.50 & 0.139 & +18.0\% [252/68] & +1.7\% [179/141] & +11.9\% [231/89] & -0.8\% [153/167] & -2.7\% [143/177] \\
0.75 & 0.104 & +16.6\% [257/63] & -0.0\% [152/168] & +12.4\% [223/96] & -0.3\% [158/162] & -8.6\% [105/215] \\
1.00 & 0.097 & +15.8\% [250/70] & +1.8\% [171/149] & +11.3\% [231/89] & -0.6\% [160/160] & -10.5\% [88/232] \\
\bottomrule\end{tabular}}\end{table}

\section{Additional tables for the LLM experiments}\label{app:llmtables}

\begin{table}[htbp]
\caption{E1 mean regret by $(d,m)$ (80 trials per cell) and E3 regret reduction by user profile ($n{=}45$ per cell).}
\label{tab:perdm}
\centering\small
\begin{tabular}{@{}lcccc@{}}
\toprule
E1 & $(25,3)$ & $(25,5)$ & $(50,3)$ & $(50,5)$ \\
\midrule
\CorrFull{} & 192.0 & 195.1 & 312.5 & 373.3 \\
\RobustCorr{} & 191.3 & 198.6 & 351.0 & 395.8 \\
Rec.-mixing & 214.1 & 219.8 & 393.3 & 543.8 \\
RandomCorr & 210.9 & 216.7 & 386.8 & 414.4 \\
CTS & 224.5 & 223.8 & 415.7 & 458.8 \\
CUCB & 546.4 & 560.6 & 933.2 & 1065.2 \\
\midrule
E3 & Diverse & Entmt. & Tech & \\
\midrule
\CorrFull{} & $+41.4\%$ & $+41.5\%$ & $+40.4\%$ & \\
\KMRefine{} & $+41.2\%$ & $+41.5\%$ & $+40.3\%$ & \\
\EdgePrune{} & $+40.5\%$ & $+39.6\%$ & $+40.0\%$ & \\
Warm-start & $-1.6\%$ & $-0.3\%$ & $-1.4\%$ & \\
\bottomrule
\end{tabular}
\end{table}

\paragraph{E2 by split.} Mean regret on the original / held-out families ($n{=}240$ each): \EdgePrune{} $558.0$ / $598.9$, \CorrFull{} $577.5$ / $598.3$, CTS $614.7$ / $617.1$; \KMRefine{} was run on both splits jointly ($573.8$).
Pilot alternatives were at or below CTS at this horizon: per-arm damping $613.6$ ($+0.4\%$), kernel decay toward $I_d$ $623.8$ ($-1.3\%$), a hierarchical-Bayes model over cluster means $630.5$ ($-2.4\%$).
\paragraph{Endogenous oracle quality.}\label{sec:endogenous} On the $d{=}25,m{=}5$ family (20 seeds), the LLM's best cluster covered $4$ of $5$ true top-$m$ slots when $\hat\mu$ came from 30 rounds of CTS warmup, but $0.15$ of $5$ on average when the same 30 rounds were round-robin, with only the numbers in the prompt differing; oracle quality is therefore endogenous to the learner's trajectory, which is why we do not model it as an exogenous constant.
\paragraph{E1 ablations.} \CorrFull{} beats RandomCorr by $12.7\%$ ($p<10^{-4}$, 228/92); a block-diagonal kernel on the LLM partition (within-cluster correlation only, no rank ordering) does not beat RandomCorr ($p=0.34$).
Under adversarial clusters (true top-$m$ grouped with the $m$ worst arms, ranked first) \CorrFull{} incurs $385.2$ against $330.7$ for CTS ($+16.5\%$ regret) and \RobustCorr{} $336.8$ ($+1.8\%$).
\Cref{sec:ablation} shows that the margin over RandomCorr is explained by kernel shape.

\paragraph{Cluster quality.}
NMI between the LLM partition and the top-$m$/rest partition over the $1{,}280$ E1 queries was $0.155\pm0.002$ (random: $0.10$), ranging from $0.09$ at $(50,3)$ to $0.23$ at $(25,5)$.
NMI rose with model capability on the validation grid (GPT-4.1-mini $0.41$, GPT-5-mini $0.44$, GPT-5.4 $0.52$).
Given \cref{sec:ablation}, the E1 gain is not read as evidence of partition quality; the NMI numbers are reported for completeness.

\paragraph{Cost.}
One cluster query is about $620$ tokens (\$0.003 on GPT-5.4); recommendation mixing used about 25 queries per E1 trial.
Per round, \CorrFull{} costs one $O(d^2)$ matrix--vector product plus a one-time $O(d^3)$ Cholesky ($<50$\,ms at $d=200$); ESCB-style index maximization over $\binom dm$ subsets is infeasible at $d=200,m=5$ without additional structure.

\clearpage
\section{Additional figures}\label{app:figs}

\begin{figure}[htbp]
\centering
\includegraphics[width=\linewidth]{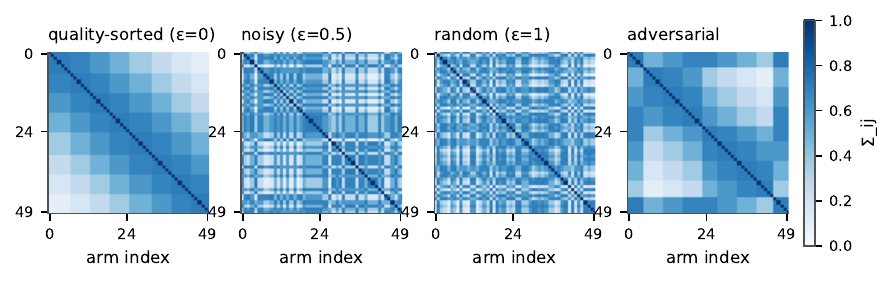}
\caption{The correlation matrix $\Sigma$ of \cref{eq:rbf} ($\rho_{\max}{=}0.7$, $\ell{=}0.5$, $K{=}8$) on a $d{=}50$ staggered instance, arms sorted by true mean, for a quality-sorted partition, a noisy one ($\varepsilon{=}0.5$), a random one, and the adversarial construction. Every matrix has unit diagonal and $\lambda_{\min}\ge1-\rho_{\max}$; only the placement of the off-diagonal mass changes, which is what \cref{sec:ablation} shows matters little on unstructured instances.}
\label{fig:kernels}
\end{figure}

\begin{figure}[htbp]
\centering
\includegraphics[width=\linewidth]{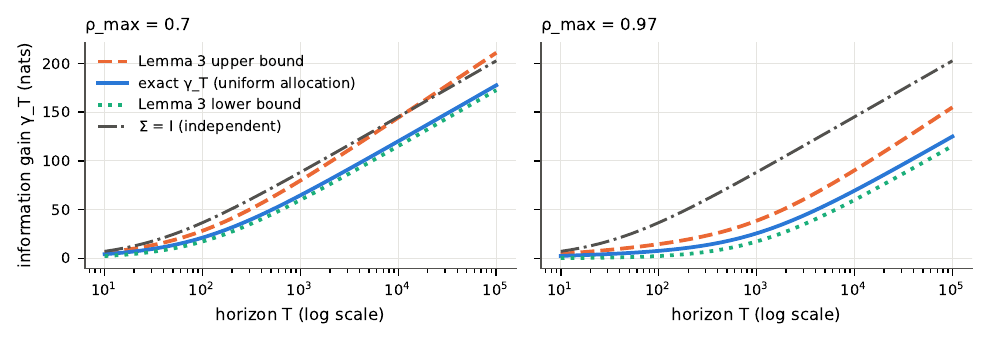}
\caption{Numerical check of \cref{lem:gamma}: exact information gain $\gamma_T$ of the kernel of \cref{fig:kernels} under a uniform allocation ($\tau^2{=}\nicefrac1{12}$, $\sigma^2{=}\nicefrac14$, $m{=}5$) against the two-term upper bound, the lower bound, and the independent-prior value. At $\rho_{\max}{=}0.7$ the curves merge with the independent case within a few hundred rounds; at $\rho_{\max}{=}0.97$ the structural regime is visibly longer, as \cref{prop:crossover} predicts.}
\label{fig:infogain}
\end{figure}

\begin{figure}[htbp]
\centering
\begin{minipage}[t]{0.48\linewidth}\centering\includegraphics[width=\linewidth]{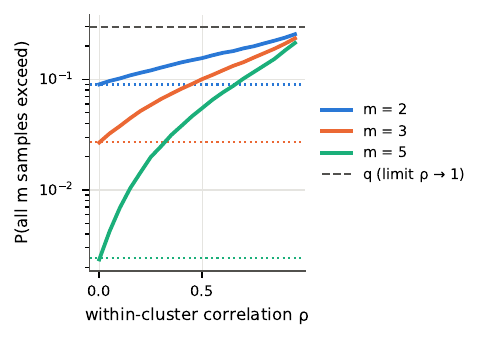}\end{minipage}\hfill
\begin{minipage}[t]{0.48\linewidth}\centering\includegraphics[width=\linewidth]{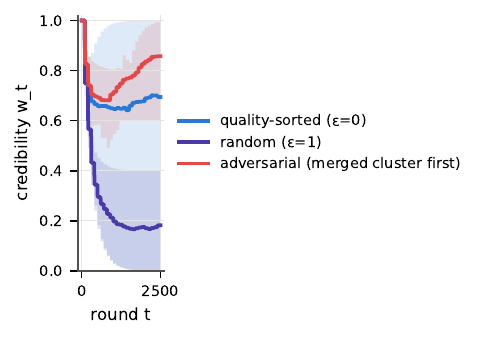}\end{minipage}
\caption{Left: \cref{lem:slepian} numerically. Probability that all $m$ equicorrelated standard Gaussian samples exceed the $q{=}0.3$ upper quantile, as a function of the correlation $\rho$ ($2\times10^5$ Monte Carlo draws); dotted lines are the independent values $q^m$, the dashed line the limit $q$. Right: credibility $w_t$ of \RobustCorr{} on the $d{=}50,m{=}5$ hard-gap family (10 seeds; band is the min--max range) for a quality-sorted partition, a random one, and the adversarial construction with the merged cluster ranked first. The decay is geometric at the checks every $\Delta_{\mathrm{check}}{=}100$ rounds as in \cref{prop:gate}, but the overlap statistic compares the oracle's best cluster with the top-$m$ of the \emph{raw} empirical means, which rarely pulled arms with $\hat\mu_i\approx1$ enter; the gate therefore also decays under a good partition (the cost visible in \cref{tab:kernelshape,tab:sweep}), and it does not fire when the merged adversarial cluster contains the true top-$m$. A lower-confidence-bound statistic would remove both effects.}
\label{fig:slepiangate}
\end{figure}

\begin{figure}[htbp]
\centering
\includegraphics[width=\linewidth]{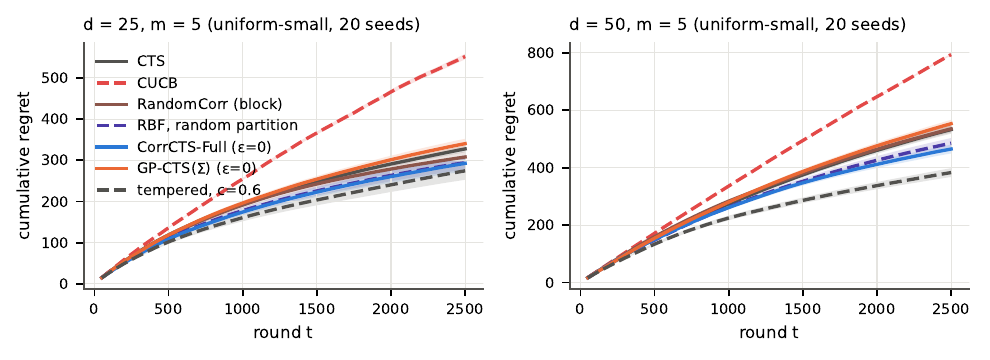}
\caption{LLM-free cumulative regret on the uniform-small family (20 seeds; bands are one standard error) for the independent re-implementation of \cref{sec:ablation}: the RBF kernel on a random partition and \CorrFull{} on the quality-sorted partition are nearly indistinguishable, the tempered sampler is lowest at this horizon, and the exact Bayesian $\GPCTS(\Sigma)$ tracks CTS.}
\label{fig:appcurves}
\end{figure}

\begin{figure}[htbp]
\centering
\includegraphics[width=\linewidth]{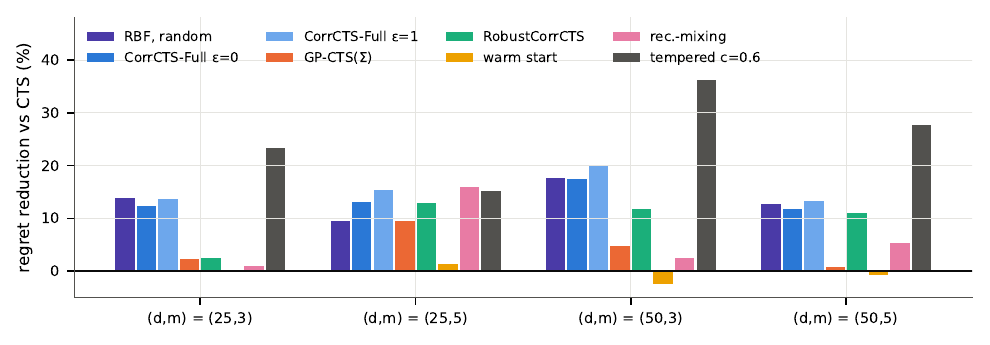}
\caption{Per-$(d,m)$ paired regret reduction versus CTS at $T{=}2{,}500$ for the samplers of \cref{sec:ablation} (80 trials per cell). The ordering is the same in every cell: tempering and the RBF samplers (whatever the partition) gain the most, the exact Bayesian sampler and the belief-injection baselines gain little or nothing.}
\label{fig:appperdm}
\end{figure}

\begin{figure}[htbp]
\centering
\includegraphics[width=0.5\linewidth]{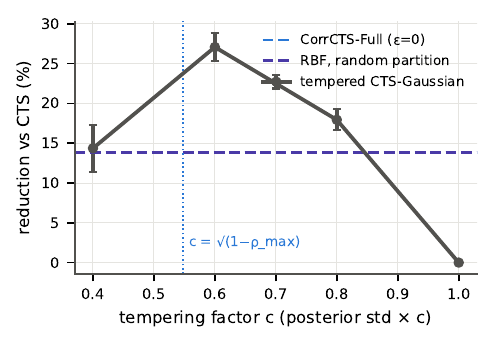}
\caption{Tempering sweep. Paired regret reduction of independent CTS-Gaussian with posterior standard deviation scaled by $c$ ($n{=}320$; $c{=}1$ is CTS-Gaussian), with the two RBF samplers as dashed references and the value $c=\sqrt{1-\rho_{\max}}$ of \cref{lem:temper}(ii) marked.}
\label{fig:apptemper}
\end{figure}

\begin{figure}[htbp]
\centering
\includegraphics[width=0.5\linewidth]{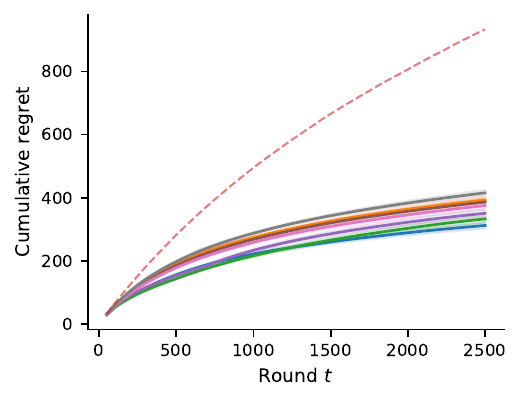}\\[2pt]
{\small \key{1f77b4}{\CorrFull{}}\quad \key{2ca02c}{\RobustCorr{}}\quad \key{ff7f0e}{Rec.-mixing}\quad \key{8c564b}{RandomCorr}\quad \key{e377c2}{block-diagonal}\quad \key{7f7f7f}{CTS}\quad \keydashed{d62728}{CUCB}}
\caption{E1 cumulative regret (LLM partitions) on the uniform-small family at $d{=}50$, $m{=}3$ (20 seeds; bands are one standard error), complementing \cref{fig:curves}.}
\label{fig:appm3}
\end{figure}

\end{document}